\documentclass{article}

\usepackage{authblk}

\usepackage{times}
\usepackage{soul}
\usepackage{url}
\usepackage[utf8]{inputenc}
\usepackage{amsmath}
\usepackage{amsthm}
\usepackage{booktabs}
\usepackage{multicol, blindtext}
\usepackage{latexsym,amsmath,amsbsy,amssymb,epsfig,lscape}
\usepackage{amsthm}
\usepackage{graphicx}
\usepackage{framed,enumerate,url,paralist}
\usepackage{longtable,tabu}

\usepackage{textcomp}
\usepackage[normalem]{ulem}
\usepackage{subcaption}
\usepackage{tikz}
\usepackage{colortbl}
\usepackage{caption}
\usepackage{soul}
\usepackage{blindtext}

\usepackage{comment}
\usepackage[switch]{lineno}

\newtheorem {definition}   {Definition}

\newtheorem {lemma}       {Lemma}
\newtheorem {proposition} {Proposition}

\newtheorem {note}          {Note}
\newtheorem {example}    {Example}

\newtheorem {remark}          {Remark}

\newcommand {\tup}[1]      {{\langle #1 \rangle}}

\newcommand{\powerset}  {\wp}
\newcommand{\ABF}         {{\sf ABF}}

\newcommand{\Free}        {{\sf Free}}
\newcommand{\Args}        {{\cal A}}

\newcommand{\F}             {{\sf F}}

\newcommand{\AF}        {{\sf AF}}
\newcommand{\I}{\mathcal{I}}

\def\addlegendimage{\csname pgfplots@addlegendimage\endcsname}
\usetikzlibrary{shapes,arrows,positioning,backgrounds,fit,calc,shapes.misc}

\newcommand{\ignore}[1] { }

\newcommand{\jesse}[1]{{\bf COMMENT}{\color{orange}#1}}

\newcommand{\chrrminv}[1]{}
\newcommand{\Ab}{Ab}

\newcommand{\baserank}{{best rank}\xspace}
\newcommand{\basescore}{{\mathsf{BestScore}}}

\title{Ranking-based Argumentation Semantics Applied to Logical Argumentation (full version)}

\author[1]{Jesse Heyninck}
\author[2,3]{Badran Raddaoui}
\author[3]{Christian Stra{\ss}er}
\affil[1]{Open Universiteit, the Netherlands}
\affil[2]{SAMOVAR, Télécom SudParis, Institut Polytechnique de Paris, France}
\affil[3]{Institute for Philosophy II, Ruhr-Universität Bochum, Germany}
\date{}

\begin{document}

\maketitle

\begin{abstract}
In formal argumentation, a distinction can be made between extension-based semantics, where sets of arguments are either (jointly) accepted or not, and ranking-based semantics, where grades of acceptability are assigned to arguments.
Another important distinction is that between abstract approaches, that abstract away from the content of arguments, and structured approaches, that specify a method of constructing argument graphs on the basis of a knowledge base.
While ranking-based semantics have been extensively applied to abstract argumentation, few work has been done on ranking-based semantics for structured argumentation. In this paper, we make a systematic investigation into the behaviour of ranking-based semantics applied to existing formalisms for structured argumentation.
We show that a wide class of ranking-based semantics gives rise to so-called culpability measures, and are relatively robust to specific choices in argument construction methods.
\end{abstract}

\section{Introduction}

A central  concept in formal argumentation \cite{Dung1995} is that of an \emph{argumentation graph}, that consists of a set of nodes interpreted as arguments, and edges interpreted as argumentative attacks. 
Within the area of formal argumentation, one can distinguish between \emph{abstract} and \emph{structured}  approaches.
Abstract argumentation abstracts away from the content of arguments and thus can be seen as taking argumentation graphs as primitive objects of study.
Structured approaches, on the other hand, specify construction methods defining how an argumentation framework is generated on the basis of a knowledge base consisting of formulas specified in a logical language.
\emph{Plausible reasoning} \cite{besnard2018review,ABH19,arieli2021logic} is a sub-class of structured argumentation formalising reasoning on the basis of defeasible premises and non-defeasible rules or logical systems such as classical logic. 

\emph{Argumentation semantics} specify the acceptability of arguments given an argumentation graph. One can distinguish between so-called \emph{extension-based} semantics, that specify which sets of arguments are jointly acceptable, and \emph{gradual} or \emph{ranking-based} semantics, that express acceptability in terms of a ranking over the arguments. Recent years have seen an intensive study of both kind of semantics in both breath and depth (see \cite{bonzon2016comparative,baroni2019fine} for some overviews) from the abstract perspective.
For extension-based semantics, the structured perspective has played a major role.
For ranking-based semantics, on the other hand, virtually no work has been carried out from the structured perspective (the only exception being \cite{amgoud2015argumentation}).
Such a gap is rather surprising, as structured argumentation still provides one of the gold standards of the construction or generation of abstract argumentation graphs.
Indeed, it is not clear whether and how the ranking-based semantics provide intuitive meaning from the logical perspective of structured argumentation. 
Furthermore, the study of ranking arguments in a structured setting could be helpful for numerous real-world applications of structured argumentation \cite{DBLP:journals/argcom/CyrasOKT21,DBLP:journals/ail/Prakken20,DBLP:conf/kr/ZengSTCLWCM20}.
For instance, in medicine and legal reasoning, it might be more appropriate to assign arguments a gradual score instead of distinguishing merely between acceptance and rejection.

In this paper, we make the first general study of ranking-based semantics applied to structured argumentation frameworks.
We 
focus on logic-based argumentation, an important and well-studied sub-class of structured argumentation. 
In particular, we answer the following questions:

\begin{enumerate}
    \item What is the logical meaning of ranking-based semantics applied to argumentation graphs generated on the basis of  strict and defeasible information?
    \item What is the influence of choices 
    underlying the construction of the argument graph (e.g.\ the form of argumentative attacks) on the outcome of ranking-based semantics?
\end{enumerate}

The main results are the following: (1) a large family of ranking-based semantics induce so-called \emph{culpability measures}, and behave better than many existing culpability measures,
(2) ranking-based semantics applied over argumentation graphs constructed using approaches from structured argumentation 
are robust under many choices considered in the literature such as different argumentative attack forms or argument construction methods, apart from the filtering out of inconsistent arguments. 

\noindent
{\bf Outline of the paper.}
In Section \ref{sec:prelim}, the relevant preliminaries on simple contrapositive assumption-based frameworks (Subsection \ref{sec:scabfs}), ranking-based argumentation semantics (Subsection \ref{sec:ranking:based:semantics}) and culpability measures (Subsection \ref{sec:inconsistency:measures}) are recalled. We review  postulates for ranking-based semantics in Section \ref{sec:ranking-based:AF:semantics:postulates}. 
In Section \ref{sec:ranking:based:are:culp}, we show that ranking-based semantics applied to argumentation frameworks obtained on the basis of simple contrapositive assumption-based frameworks induce culpability measures. In Section \ref{sec:other:structured:arg}, we look at other structured argumentation formalisms.
We review related work in Section \ref{sec:rel:work} and conclude in  Section \ref{sec:concl}.

\section{Preliminaries}
\label{sec:prelim}
This section contains background material about simple contrapositive assumption-based argumentation frameworks (more details can be found in~\cite{AH21,HA20}), 
ranking-based semantics, and
culpability measures.

\subsection{Simple Contrapositive Assumption-based Frameworks}
\label{sec:scabfs}

We shall denote by ${\cal L}$ an arbitrary propositional language.
Atomic formulas in 
${\cal L}$ are denoted by $p,q,r,$ etc., arbitrary  ${\cal L}$-formulas (atomic or compound) are denoted by $\psi,\phi,\delta,$ etc., and sets of 
formulas  in ${\cal L}$ are denoted by $\Gamma$, $\Delta$, $\Theta,$ etc. (possibly primed or indexed). 
$\powerset(\mathcal{L})$ denotes
the power-set of ${\cal L}$.

Instead of restricting ourselves to one specific logic, we merely assume some common properties of logics in order to keep our results as general as possible.

\begin{definition}[{\bf Logic}]
\label{def:logic}
{ A {\em logic\/} for a language ${\cal L}$ is a pair ${\mathfrak L} = \tup{{\cal L},\vdash}$, where 
$\vdash$ is a (Tarskian) consequence relation for ${\cal L}$, that is, a binary relation among sets of formulas 
and individual formulas in ${\cal L}$, satisfying the following conditions:  
\begin{itemize}
    \item [{\sl Reflexivity\/}:] if $\psi \in \Gamma$, then $\Gamma \vdash \psi$.  
    \item [{\sl Monotonicity\/}:] 
             if $\Gamma \vdash \psi$ and $\Gamma \subseteq \Gamma^\prime$,
             then $\Gamma^\prime \vdash \psi$.
    \item [{\sl Transitivity\/}:]
             if $\Gamma \vdash \psi$ and $\Gamma^\prime,\psi \vdash \phi$
             then $\Gamma\cup\Gamma^\prime \vdash \phi$.
\end{itemize}

} \end{definition}

In what follows, we shall assume that the language ${\cal L}$ contains at least the following connectives and constant:
\begin{description}
    \item [] a {\em $\vdash$-negation\/} $\neg$, satisfying: $p \not\vdash \neg p$ and $\neg p \not\vdash p$
            (for every atomic $p$).
    \item [] a {\em $\vdash$-conjunction\/} $\wedge$, satisfying: $\Gamma \vdash \psi \wedge \phi$ iff 
            $\Gamma \vdash \psi$ and $\Gamma \vdash \phi$.

    \item [] a {\em $\vdash$-falsity\/} ${\sf F}$, satisfying: ${\sf F} \vdash \psi$ for every formula $\psi$.\footnote
            {Particularly, $\F$ is not a standard atomic formula, since $\F\vdash\neg\F$.} 
\end{description}
Our results only assume the above connectives with the respective properties. This means that these results apply to any language with additional connectives or additional properties assumed on the connectives.
We abbreviate $\{\neg\psi \mid \psi \in \Gamma\}$ by $\neg\Gamma$, and when $\Gamma$ is finite 
we denote by $\bigwedge\!\Gamma$ the conjunction of all the formulas in $\Gamma$.  

\medskip
Given a logic ${\mathfrak L} = \tup{{\cal L},\vdash}$, the $\vdash$-closure of a set $\Gamma$ of ${\cal L}$-formulas is 
defined by $Cn_{\vdash}(\Gamma)=\{\psi\mid \Gamma\vdash\psi\}$. When $\vdash$ is clear from the context or
arbitrary, we will sometimes simply write $Cn(\Gamma)$. 
Now,
\begin{itemize}
     \item We shall say that $\psi$ is a {\em $\vdash$-theorem\/} if $\psi \in Cn_{\vdash}(\emptyset)$, and that $\Gamma$ is 
              {\em $\vdash$-consistent\/} if $Cn_{\vdash}(\Gamma) \neq {\cal L}$ 
              (thus, $\Gamma \not\vdash \F$, otherwise, we say that $\Gamma$ is {\em $\vdash$-inconsistent}).
\item ${\mathfrak L}$ is called {\em explosive\/}, if for every ${\cal L}$-formula $\psi$ the set $\{\psi,\neg\psi\}$ is $\vdash$-inconsistent.
\item ${\mathfrak L}$ is called {\em contrapositive\/}, if (i)~$\vdash \neg \mathsf{F}$, and (ii)~for every non-empty set
$\Gamma$ and $\psi$ it holds that $\Gamma \vdash \neg\psi$ iff  either $\psi = {\sf F}$ or for every $\phi \in \Gamma$ we have that 
$(\Gamma \setminus \{\phi\}) \cup \{ \psi \} \vdash \neg\phi$.
 \item ${\mathfrak L}$ is called {\sl uniform\/} if
for any $\Gamma\cup\Gamma'\cup\{\phi\}\subseteq {\cal L}$, if $\Gamma'\not\vdash{\sf F}$ and $\Gamma\cup\{\phi\}$ and $\Gamma'$ have no atoms in common, it holds that: $\Gamma\vdash \phi$ iff  $\Gamma\cup\Gamma'\vdash \phi$.
\end{itemize}

\paragraph{Remark.}
{ Perhaps the most notable example of a logic which is explosive, contrapositive and uniform, is classical logic (in short, ${\sf CL}$)\footnote{In fact, by the {\L}os-Suszko theorem \cite{Los_Suszko_1958} any structural logic with a characteristic matrix is uniform.}.
Intuitionistic logic and standard modal logics are other examples 
of well-known formalisms having these properties. }

\begin{note}
\label{note:inconsistent-set}
{ A useful property of an explosive logic ${\mathfrak L} = \tup{{\cal L},\vdash}$ is that for every set $\Gamma$ of ${\cal L}$-formulas
and any ${\cal L}$-formulas $\psi$ and $\phi$, if $\Gamma \vdash \psi$ and $\Gamma \vdash \neg\psi$, then $\Gamma \vdash \phi$.

} \end{note}

The following family of assumption-based argumentation frameworks~\cite{Bondarenko1997} has been shown 
to be a useful setting for argumentative reasoning~\cite{HA20}.
In more detail, a simple contrapositive assumption-based framework 
is a knowledge base consisting of a set of strict premises $\Gamma$, a set of defeasible premises $\Ab$, a logic $\mathfrak{L}$, and a contrariness operator specifying which conflicts are taken into account.

\begin{definition}[{\bf Simple contrapositive ABFs}]
\label{deductivesystems}
{
An \emph{assumption-based framework\/} (ABF, for short) is a tuple ${\sf ABF}= \tup{\mathfrak{L}, \Gamma, \Ab, \sim}$ 
where:
\begin{itemize}
\item ${\mathfrak L} = \tup{{\cal L},\vdash}$ is a propositional Tarskian logic.
\item $\Gamma$ (the \emph{strict assumptions\/}) and $\Ab$ (the \emph{candidate/defeasible assumptions\/}) 
         are distinct (countable) sets of $\mathcal{L}$-formulas, where the former is assumed to be $\vdash$-consistent 
         and the latter is assumed to be nonempty. 
\item $\sim\::\Ab\rightarrow \powerset(\mathcal{L})$ is a \emph{contrariness operator\/}, assigning a finite set of ${\cal L}$-formulas 
         to every defeasible assumption in $\Ab$, such that for every consistent and non-theorem formula $\psi \in \Ab $ 
         it holds that $\psi \not\vdash \bigvee\!\sim\!\psi$ and $\bigvee\!\sim\!\psi \not\vdash \psi$. 
\end{itemize}
A {\em simple contrapositive\/} $\ABF$ is an 
assumption-based framework 
${\sf ABF} = \tup{\mathfrak{L}, \Gamma, \Ab, \sim}$, 
where $\mathfrak{L}$ is an explosive, contrapositive and uniform logic, and $\sim\!\psi=\{\lnot\psi\}$.\footnote{Notice that uniformity is not required explicitly in \cite{HA18}. To avoid clutter, we choose to adapt the definition here instead of using ``simple contrapositive uniform assumption-based frameworks.''}
} \end{definition}

\begin{definition}
\label{def:non-proiritized-attack}
{ Let $\ABF=\tup{\mathfrak{L},\Gamma,\Ab,\sim}$ be an assumption-based framework, $\Delta,\Theta \subseteq \Ab$, 
and $\psi \in \Ab$. 
Then, 
$\Delta$ \emph{attacks\/} $\psi$ iff $\Gamma\cup\Delta \vdash \phi$ for some $\phi \in \:\sim\!\psi$. 
Then, $\Delta$ attacks $\Theta$ if $\Delta$ attacks some $\psi \in \Theta$.
} \end{definition}

The following useful notions 
are used in the paper:

\begin{definition}\label{def:mcs}
 Let $\ABF=\tup{\mathfrak{L},\Gamma,\Ab,\sim}$.
 A set $\Delta\subseteq\Ab$ is {\em maximal consistent\/} 
in $\ABF$, if (i)~$\Gamma\cup\Delta \not\vdash \F$ and (ii)~$\Gamma\cup\Delta' \vdash \F$ for every 
$\Delta \subsetneq \Delta'\subseteq\Ab$.
The set of all the maximal consistent sets in $\ABF$ is denoted ${\sf MCS}(\ABF)$.
A set $\Delta\subseteq\Ab$ is {\em minimal inconsistent} in $\ABF$, if  (i)~$\Gamma\cup\Delta \vdash \F$ and (ii)~$\Gamma\cup\Delta' \not\vdash \F$ for every 
$\Delta \supsetneq \Delta' \subseteq\Ab$.
The set of all the minimal inconsistent sets in $\ABF$ is denoted ${\sf MIC}(\ABF)$.
\\
We define the set of {\em free formulas\/} in $\ABF$ as ${\sf FREE}(\ABF)=\Ab\setminus\bigcup {\sf MIC}(\ABF)$.
\end{definition}

For a set of formulas $\Ab\subseteq {\cal L}$, we will denote the set of all classically maximal consistent subsets of $\Ab$, defined as  ${\sf MCS}(\tup{\mathsf{CL},\emptyset,Ab,\sim})$, by ${\sf MCS}(Ab)$.
Similarly for the minimal inconsistent subsets (${\sf MIC}(\Ab)$) and the free formulas (${\sf FREE}(\Ab)$) of $\Ab$.

\begin{example}
\label{ex:ABF}
Let $\mathfrak{L} = {\sf CL}$,
$\Gamma=\emptyset$, $Ab=\{p,\lnot p, q\}$, and $\sim\!\psi = \{\neg\psi\}$.
An attack diagram for this ABF
is shown in Figure~\ref{figure:ex:ABF}. 
In fact, an attack diagram consists of all subsets of assumptions, and the attacks according to Definition \ref{def:non-proiritized-attack} between these assumptions.
Note that since in ${\sf CL}$
inconsistent sets of premises imply {\em any\/} conclusion, 
the classically inconsistent set $\{p,\neg p,q\}$ attacks all the other sets in the diagram
(for instance, $\{p,\neg p,q\}$ attacks $\{q\}$, since $\{p,\neg p,q \}\vdash \neg q$).

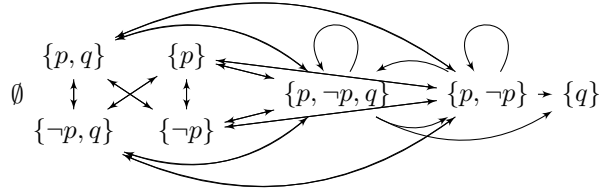
\begin{figure}[!h]
\centering
\begin{tikzpicture}[scale=0.5]
\tikzset{edge/.style = {->,> = latex'}}

\node (q) at (7.5,0) {$\{q\}$};
\node (p&-p) at (5,0) {$\{p,\lnot p\}$};
\node (p&-pq) at (1,0) {$\{p,\lnot p,q\}$};

\node (p) at (-3,1) {$\{p\}$};
\node (-p) at (-3,-1) {$\{\lnot p\}$};
\node (pq) at (-6,1) {$\{ p,q\}$};
\node (-pq) at (-6,-1) {$\{ \neg p,q\}$};
\node (empty) at (-7.5,0) {$\emptyset$};

\draw[edge,loop=20, out=60,in=120,looseness=7,scale=0.6] (p&-p) to (p&-p);
\draw[edge,loop=20, out=60,in=120,looseness=7,scale=0.6] (p&-pq) to (p&-pq);

\draw[edge] (p&-p) to (q);
\draw[edge] (p&-p) to (p);
\draw[edge] (p&-p) to (-p);
\draw[edge, bend right=30] (p&-p) to (pq);
\draw[edge, bend left=30] (p&-p) to (-pq);
\draw[edge] (p) to (p&-p);
\draw[edge] (-p) to (p&-p);
\draw[edge, bend left=30] (pq) to (p&-p);
\draw[edge, bend right=30] (-pq) to (p&-p);
\draw[edge, bend right=30] (p&-pq) to (q);
\draw[edge] (p&-pq) to (p);
\draw[edge] (p&-pq) to (-p);
\draw[edge, bend right=30] (p&-pq) to (pq);
\draw[edge, bend left=30] (p&-pq) to (-pq);
\draw[edge] (p) to (p&-pq);
\draw[edge] (-p) to (p&-pq);
\draw[edge, bend left=30] (pq) to (p&-pq);
\draw[edge, bend right=30] (-pq) to (p&-pq);
\draw[edge, bend right=30] (p&-pq) to (p&-p);
\draw[edge, bend right=30] (p&-p) to (p&-pq);

\draw[edge] (p) to (-p);
\draw[edge] (p) to (-pq);
\draw[edge] (-p) to (p);
\draw[edge] (-p) to (pq);
\draw[edge] (-pq) to (pq);
\draw[edge] (pq) to (-pq);
\draw[edge] (-pq) to (p);
\draw[edge] (pq) to (-p);

\end{tikzpicture}

\caption{\fontsize{9}{1}An attack diagram for Example~\ref{ex:ABF}}\label{figure:ex:ABF}
\end{figure}
 \end{example}

\subsection{Ranking-based Argumentation Semantics}
\label{sec:ranking:based:semantics}

Ranking-based semantics have been mainly studied from the abstract perspective, that focuses on abstract argumentation frameworks:

\begin{definition}
An \emph{argumentation framework}  $\AF=(\Args,\rightarrow)$ consists of a set of arguments $\Args$ and an attack relation ${\rightarrow} \subseteq \Args\times\Args$.
Where $a\in \Args$, $a^{-}=\{b\in\Args\mid b\rightarrow a\}$.
\end{definition}

Notice that for any ${\sf ABF} = \langle \mathfrak{L},\Gamma,Ab,\sim\rangle$, $(\wp(\Ab),\rightarrow)$
is an argumentation framework, where $\rightarrow$ is defined as in Definition \ref{def:non-proiritized-attack}.
In other words, an assumption-based argumentation gives a method of how to construct an argumentation framework. We now discuss ranking-based semantics, that allow to determine the \emph{acceptability} of arguments  for a given $\AF$.

There has been a wide range of ranking-based semantics proposed in the literature to rank order arguments from the most acceptable to the weakest arguments \cite{bonzon2016comparative}.
Following Baroni, Rago and Toni \cite{baroni2019fine}, we consider ranking-based semantics as mappings $\sigma$ of arguments to a set $\mathbb{I}$ preordered by $\leq$ relative to an argumentation framework $\AF=(\Args,\rightarrow)$.
Given $a,b\in\Args$, $\sigma(a) \geq \sigma(b)$ means that the argument $a$ is at least as acceptable as the argument $b$ according to $\sigma$.
For simplicity, we assume the image of $\sigma$ is always $\mathbb{R}_{\geq0}$.
In the remainder of the paper, we will denote $\sigma_{\AF}$ as the ranking-based semantics relative to $\AF$.
A high ranking $\sigma_{\AF}(a)$ represents a high quality of an argument $a$ w.r.t.\ $\AF$ according to $\sigma$.

As an illustration of a ranking-based semantics, we look at the \emph{categorizer semantics}, first presented in \cite{besnard2001logic} for acyclic frameworks but generalized to any argumentation framework in \cite{amgoud2017acceptability}. This semantics formalizes the idea that the number of attackers decreases the quality of an argument. Formally, it works by  taking the inverse of the sum of the score of the attacking arguments iteratively. 

\begin{definition}
Given an $\AF=(\Args,\rightarrow)$ and $a\in\Args$:
\begin{align}
{\tt cat}_0(a)&=&1\\
{\tt cat}_{i+1}(a)&=& \frac{1}{1+\Sigma_{b\in a^{-}}{\tt cat}_i(b)}
\end{align}
\end{definition}

It has been shown in \cite{amgoud2017acceptability}  that ${\tt cat}_i$ converges as $i$ approaches infinity.
We therefore define $\sigma^{\tt cat}_{\AF}(a)=\lim_{i\to\infty}{\tt cat}_i(a)$.

\begin{example}
Consider the argumentation framework from Figure \ref{figure:ex:ABF}. It can be verified (for example using the Tweety-library \cite{thimm2014tweety}) that:
$\sigma_{\AF}^{\tt cat}(\emptyset)=1$,
$\sigma_{\AF}^{\tt cat}(\{q\})=0.71$, $\sigma_{\AF}^{\tt cat}(\{p\})=\sigma_{\AF}^{\tt cat}(\{\lnot p\})=0.52$, and $\sigma_{\AF}^{\tt cat}(\{p,\lnot p,q\})=0.41$.
\end{example}

\subsection{Culpability Measures}
\label{sec:inconsistency:measures}

A fundamental research topic in reasoning with conflicting information is the measure of inconsistencies. A lot of work focuses on \emph{inconsistency measures} that allow to quantify \emph{how inconsistent} \cite{DBLP:conf/kr/JabbourMRSS16,Thimm:2019d} a \emph{set} of formulas (typically propositional) is.

On  the other hand, 
\emph{culpability measures} quantify the degree of responsibility of a formula in making a set of formulas inconsistent.
In more detail, a culpability measure for a knowledge base $\Ab\subseteq {\cal L}$ is a function ${\cal C}_{\Ab}: \Ab \rightarrow \mathbb{R}_{\geq 0}$ that satisfies the following postulates:

\begin{description}
\item[Consistency.] If $\Ab$ is consistent, then ${\cal C}_{\Ab}(\phi)=0$ for every $\phi\in\Ab$.

\item[Blame.] If $\phi\in \bigcup{\sf MIC}(\Ab)$, then ${\cal C}_{\Ab}(\phi)>0$.
\end{description}
Thus, \emph{consistency} states that the culpability measure of premises in a consistent set are minimal (i.e.\ in consistent sets, no formula is culpable for inconsistency), whereas \emph{blame} states that any formula involved in an inconsistency will receive a non-zero culpability degree.

In \cite{hunter2008measuring}, culpability measures are defined as \emph{inconsistency values}\footnote{To comply with the notation in this paper, higher inconsistency values intuitively represent a higher degree of inconsistency, in contradistinction to the work of  \cite{hunter2008measuring}.}.
The following properties for inconsistency values are defined in \cite{hunter2008measuring} (we adapt them to our notation).\footnote{In  \cite{hunter2008measuring}, other properties of culpability measures are discussed, but they describe the relation between inconsistency and culpability measures, and are therefore not relevant to our study.}

\begin{description}
\item[Free Formula.]  If $\phi\in \Free(\Ab)$\footnote{Recall that ${\sf FREE}(\Ab)={\sf Free}(\tup{\mathsf{CL},\emptyset,Ab,\sim}) = \Ab\setminus\bigcup {\sf MIC}(\Ab)$.}, then ${\cal C}_{\Ab }(\phi)\leq {\cal C}_{\Ab}(\psi)$ for any $\psi\in\Ab$.\footnote{In \cite{hunter2008measuring}, this property is called \emph{minimality}, and it is required that a free formula evaluates to 0. We use a weaker condition that allows to differentiate between the score of $p$ in $\{p,q,\lnot q\}$ and in $\{p,q\}$.}

\item[Dominance.] Where $\phi, \psi \in \Ab$, 
if $\phi\vdash \psi$ and $\phi \not\vdash {\sf F}$, then ${\cal C}_{\Ab}(\phi)\geq {\cal C}_{\Ab}(\psi)$.
\end{description}

Thus, \emph{free formula} states that formulas not involved in any inconsistency are at most as culpable as any other formula, and \emph{dominance} states that logically stronger formulas are at least as culpable as their weaker counterparts.

By ways of illustration, we consider the following culpability measures ${\cal C}^d_{Ab}$, ${\cal C}^{\ast}_{\Ab}$ and ${\cal C}^{c}_{\Ab}$ from \cite{hunter2008measuring}:\footnote{In our presentation, unlike in \cite{hunter2008measuring}, the measures are normalized to ensure a mapping to $[0,1]$. 
We follow the convention that $\sum\emptyset=0$ which means that ${\cal C}^{c}_{\Ab}(\phi)=0$ for any free formula $\phi$. For ${\cal C}^{\ast}_{\Ab}$ and ${\cal C}^{c}_{\Ab}$, the knowledge base is supposed to be finite.}
\begin{align*}
{\cal C}^d_{Ab}(\phi)&=& \begin{cases}
    1 & \mbox{if }\phi\in \bigcup{\sf MIC}(\Ab)\\
    0 & \mbox{otherwise}
\end{cases}\\
{\cal C}^{\ast}_{\Ab}(\phi)&=&\frac{| \{\Delta\in {\sf MIC}(\Ab)\mid \phi\in\Delta\}|}{| {\sf MIC}(\Ab)|}\\
{\cal C}^{c}_{\Ab}(\phi)&=& 
    \displaystyle\left(\sum_{\Delta\in {\sf MIC}(\Ab)\mbox{ and } \phi\in\Delta}\frac{1}{|\Delta|} \right)\frac{1}{\sum_{\Delta\in {\sf MIC}(\Ab)}|\Delta|} 
\end{align*}

\begin{example}
Consider the knowledge base $\Gamma=\{p\land \lnot p, q,r,\lnot q\lor \lnot r,s\}$.
The minimal inconsistent sets of $\Gamma$ are ${\sf MIC}(\Gamma)=\{\{p\land \lnot p\},\{ q,r,\lnot q\lor \lnot r\}\}$.
This means that ${\cal C}^d_\Gamma(s) = {\cal C}^{\ast}_\Gamma(s)=0$, and  ${\cal C}^d_\Gamma(\phi) = 1$ and
${\cal C}^{\ast}_\Gamma(\phi)=0.5$ for any $\phi\in\Gamma\setminus\{s\}$.
Also, we have ${\cal C}^{c}_\Gamma(s) = 0$,
${\cal C}^{c}_\Gamma(p\land \lnot p) = \frac{1}{4}$ and
${\cal C}^{c}_\Gamma(\phi)=\frac{1}{12}$ for any $\phi\in\Gamma\setminus\{p\land \lnot p, s\}$.
\end{example}

\section{Postulates for Ranking-based Semantics}\label{sec:ranking-based:AF:semantics:postulates}
In this section, as a first illustration of the difference between abstract and structured argumentation in the context of ranking-based semantics, we look at some common postulates proposed for ranking-based semantics and evaluate their usefulness in a structured setting.
We will see that, even though these postulates are of limited meaning in the structured argumentation, they can be easily adapted to a variant that is more meaningful in the structured setting. We restrict attention to postulates that will be relevant in the rest of the paper, and leave an exhaustive study of the whole range of postulates for ranking-based semantics (as found in e.g.\ \cite{baroni2019fine}) for future work.

An important postulate in the literature on ranking-based semantics is the principle of \emph{void precedence} \cite{amgoud2013ranking}.
It ensures that unattacked arguments are maximally preferred.

\begin{definition}
A semantics $\sigma$ satisfies \emph{void precedence} iff   $a^{-}\neq\emptyset$ and $b^{-}=\emptyset$ implies $\sigma_{\AF}(a)<\sigma_{\AF}(b)$.
\end{definition}

When moving to logic-based argumentation, void precedence is often trivially satisfied as there are virtually no unattacked arguments (except for $\emptyset$) due to inconsistent arguments ``contaminating'' \cite{caminada2005contamination} innocent bystanders:

\begin{example}
 Consider again Example~\ref{ex:ABF}. We see that $\{p,\lnot p\}$ attacks $q$ due to logical explosion.
\end{example}

It is reasonable to expect an innocent bystander to be ranked higher than other arguments. We thus obtain the following postulate. We first introduce the notation $\sigma_\ABF$ to denote the ranking-based semantics $\sigma_{\sf AF}$ for ${\sf AF}=(\wp(Ab),\rightarrow)$, where $\rightarrow$ is defined as in Definition \ref{def:non-proiritized-attack}.

\begin{definition}
A semantics $\sigma$ satisfies \emph{logical void precedence} iff  for any ABF ${\sf ABF}=\tup{\mathfrak{L}, \Gamma,Ab, \sim}$, if $\Theta \subseteq \Free({\sf ABF})$, then $\sigma_{\sf ABF}(\Theta) \geq \sigma_{\sf ABF}(\Delta)$ for any $\emptyset \neq \Delta\subseteq Ab$.
\end{definition}
Notice that logical void precedence immediately implies that every free formula is ranked at least as high as any other formulas.

Logical void precedence is satisfied by any ranking-based semantics that satisfies the following property of monotony:

\begin{definition}
A semantics $\sigma$ satisfies \emph{monotony} iff  $a^{-}\subseteq b^{-}$ implies $\sigma_{\AF}(a)\geq\sigma_{\AF}(b)$.
\end{definition}

In Remark \ref{remark:mootony}, we give some examples of ranking-based semantics that satisfy monotony.
 Monotony of a ranking-based semantics implies logical void precedence.
\begin{proposition}\label{prop:void:precedence}
For simple contrapositive ABFs, any ranking-based semantics that satisfies monotony satisfies logical void precedence.
\end{proposition}
\begin{proof}
Suppose $\Theta \subseteq \Free({\sf ABF})$. Let $\emptyset \neq \Delta \subset \Ab$ be arbitrary. 
We show that every attacker $\Phi \subseteq \Ab$ of $\Theta$ also attacks $\Delta$ and so, with monotony, $\sigma_{\AF}(\Theta) \ge \sigma_{\AF}(\Delta)$.
Note that $\Gamma\cup \Phi \vdash {\sim} \psi$ for some $\psi \in \Theta$.  
Since $\psi \in \Free({\sf ABF})$, $\Gamma \cup \Phi \vdash {\sf F}$ and so $\Gamma\cup \Phi \vdash {\sim} \delta$ for any $\delta \in \Delta$. So, $\Theta$ also attacks $\Delta$.
\end{proof}

Another property that will be useful below is \emph{counter-transitivity} \cite{amgoud2013ranking}, which makes use of the following comparison of sets of arguments:

\begin{definition}
Let a ranking-based semantics $\sigma$ 
and an argumentation framework ${\sf AF}=(\Args, \rightarrow)$ be given.
For any $S_1,S_2\subseteq \Args$, $S_1\succeq_\sigma S_2$ is a \emph{group comparison} if there exists an injective mapping $f:S_2\to S_1$ s.t.\ for every $a\in S_2$, 
$\sigma_{\sf AF}(f(a)) \geq \sigma_{\sf AF}(a)$.
\end{definition}

\begin{example}\label{ex:ABF:burden:based}
Consider again Example \ref{ex:ABF}. As $\{q\}^{-}\subset \{p,\lnot p\}^{-}$, the identity function ${\sf id}$ is an injective mapping s.t.\ (for any semantics $\sigma$), $\sigma_{\sf AF}(f(\Delta)) \geq \sigma_{\sf AF}(\Delta)$ for every $\Delta\in \{q\}^{-}$ and thus $ \{p,\lnot p\}^{-}\succeq_\sigma\{q\}^{-}$.
\end{example}

Thus, a group comparison allows expressing that one set $S_1$ is better than another one $S_2$ in the sense that for every element of $S_2$ we can find an element of $S_1$ that is ranked better or equal according to $\sigma$.
We now recall the property of \emph{counter-transitivity} \cite{amgoud2013ranking} that expresses that an argument with numerous and better attackers is ranked lower:

\begin{definition}
A semantics $\sigma$ satisfies \emph{counter-transitivity} if for any ${\sf AF}=(\Args, \rightarrow)$ and any $a,b\in \Args$, $b^{-}\succeq_\sigma a^{-}$ implies $\sigma_{\sf AF}(a) \geq \sigma_{\sf AF}(b)$.
\end{definition}

It has been shown in \cite{bonzon2016comparative} that the categoriser, discussion-based \cite{amgoud2013ranking} and the burden-based semantics satisfy the postulate of counter-transitivity.
Notice that counter-transitivity implies monotony. 

\begin{remark}\label{remark:mootony}
The postulate of monotony is not generally used in the literature on postulates for ranking-based semantics\footnote{The property of (strict) monotony is studied in \cite{DBLP:conf/aaai/AmgoudD21} (see principle 2, page 5).}.
However, several postulates that imply monotony have been studied. We already noted the postulate of counter-transitivity. 
Another example is the postulate of \emph{cardinality precedence} \cite{bonzon2016comparative}, satisfied among others by the categoriser, discussion-based \cite{amgoud2013ranking} and the burden-based semantics.  A second example is the \emph{general principle 7} \cite{baroni2019fine} which is satisfied by e.g.\ the categoriser semantics, inverse equational systems \cite{gabbay2012equational}, LocSum \cite{amgoud2008bipolarity}, and social models \cite{leite2011social}.
\end{remark}

\section{Ranking-based Semantics as  Culpability Measures}
\label{sec:ranking:based:are:culp}

In this section, we look at the logical meaning of ranking-based semantics. 
More precisely, we ask the following question: when applying ranking-based semantics to an argumentation framework generated using a structured argumentation methodology, \textit{what is the logical meaning of the ranking over arguments?} 
It turns out that a ranking-based semantics has a clear logical meaning: it induces an inversed culpability measure.
Thus, the only difference between a culpability measure and a ranking-based semantics is the ``direction'': high ranking scores equal low culpability values and vice versa.

Now, in more formal detail, as is common in structured argumentation, we consider an $\ABF$ as a knowledge base that consists of strict and defeasible premises, equipped with a logic.
We note that the fact that we allow for strict premises and a wide choice of logics is a generalization w.r.t. the assumptions usually made for culpability measures (i.e., existing measures are always defined for a specific logic, in most cases classical logic.).
Then, ranking-based semantics are applied to the argumentation framework generated on the basis of the \ABF.
Afterwards, we look at the arguments $\Delta\subseteq Ab$ (i.e.\ sets of defeasible premises) and will see that the rankings over these arguments induce a culpability measure.
As we allow for strict premises and a wide variety of logics, the postulates from Subsection \ref{sec:inconsistency:measures}  have to be generalized, as we do in Proposition \ref{prop:ranking:is:a:culpability}. We also generalize the postulates to sets of defeasible assumptions instead of single assumptions.

In culpability measures, the value $0$ is seen as a ``\baserank'', i.e.,\ the rank assigned to the maximally plausible formulas.
For the work done below, it is therefore often useful to assume a \baserank w.r.t.\ a ranking-based semantics.
\begin{definition}
A semantics $\sigma$ satisfies \emph{void \baserank} if there is an $r\in \mathbb{R}_{\geq 0}$ s.t.\ for every $\AF=({\cal A},\to)$ and $a\in {\cal A}$ with $a^{-}=\emptyset$, $\sigma_{{\sf AF}}(a)=r$. 
In that case, we say $r=\mathsf{\basescore}(\sigma)$.
\end{definition}

It can be easily observed that the void \baserank property is satisfied by the categoriser semantics, as unattacked arguments are assigned the rank $1$ at any step in the computation.
Likewise, this property holds for any instantiation of the social models from \cite{leite2011social}, the burden-based semantics and the discussion-based semantics \cite{amgoud2013ranking}.

The next result shows that a ranking-based semantics is a culpability measure satisfying dominance and free formula postulates.
Notice that all postulates of culpability measures are generalized to the setting of simple contrapositive ABFs, which have as a special case the culpability measures as defined in Section \ref{sec:inconsistency:measures} (with $\ABF=({\sf CL},\emptyset, Ab,\sim)$).

\begin{proposition}\label{prop:ranking:is:a:culpability}
Any ranking-based semantics that satisfies monotony, void precedence and void \baserank satisfies the following properties (for a simple contrapositive ABF $\ABF=({\frak L},\Gamma, Ab,\sim)$):

\begin{description}

\item[Freeness.] If $\Psi\subseteq\Free(\ABF)$, then $\sigma_\ABF(\Psi)\geq \sigma_\ABF(\Delta)$ for any $\emptyset\neq\Delta\subseteq Ab$.

\item[Dominance.] If $\Gamma\cup\{\phi\}\vdash \psi$ and $\Gamma\cup\{\phi\}\not\vdash {\sf F}$, then $\sigma_{\sf ABF}(\{\psi\})\geq \sigma_{\sf ABF}(\{\phi\})$.

\item[Blame.] If $\phi\in\bigcup{\sf MIC}(\ABF)$, then $\sigma_\ABF(\{\phi\}) <  \basescore(\sigma)$.

\item[Consistency.] If $\Gamma\cup Ab\not\vdash {\sf F}$, then $\sigma_\ABF(\Phi)= \basescore(\sigma)$ for every $\Phi \subseteq \Ab$.

\end{description}
\end{proposition}

\begin{proof}
{\bf Freeness}: Follows immediately from Proposition \ref{prop:void:precedence}.
{\bf Dominance}: Consider $\phi,\psi\in Ab$ s.t.\ $\Gamma\cup \{\phi\}\vdash \psi$ and  $\Gamma\cup\{\phi\}\not\vdash {\sf F}$.
Suppose some $\Delta\subseteq Ab$ attacks $\{\psi\}$, i.e.,\ $\Gamma\cup\Delta\vdash {\sim}\psi$.
So  with contraposition, $\Gamma\cup\{\psi\}\vdash {\sim} \bigwedge \Delta$, and thus, $\Gamma\cup\{\phi\}\vdash {\sim}\bigwedge\Delta$ and thus, $\Gamma\cup\Delta\vdash {\sim} \phi$.
This means $\{\psi\}^{-}\subseteq \{\phi\}^{-}$ and thus with monotony, $\sigma_{\sf ABF}(\{\phi\})\leq \sigma_{\sf ABF}(\{\psi\})$.
{\bf Blame}: Suppose $\phi\in \bigcup{\sf MIC}({\sf ABF})$. Then, there is some $\Delta\in {\sf MIC}( {\sf ABF})$ s.t.\ $\phi\in\Delta$. As $\Gamma\cup\Delta\vdash {\sim}\phi$, $\{\phi\}^{-}\neq\emptyset$ and thus with void precedence, void \baserank and since $\emptyset^{-}=\emptyset$, $\sigma_{\sf ABF}(\phi) <\sigma_{\sf ABF}(\emptyset)=\basescore(\sigma)$.
{\bf Consistency}: Suppose $\Gamma\cup Ab\not\vdash {\sf F}$.
Then, there are no attacks between sets in $(\wp(Ab),\rightarrow)$, and thus with void \baserank, we have that $\sigma_\ABF(\Phi) =\basescore(\sigma)$ for every $\Phi\subseteq Ab$. \qedhere
\end{proof}

Besides establishing that ranking-based semantics induce culpability measures, we show that these induced culpability measures perform better than  existing culpability measures based on minimal inconsistent subsets.
This is demonstrated using the following example:

\begin{example}
Let $Ab=\{p\land \lnot p, q,\lnot q\land r, \lnot q\land s\}$.
Then, ${\cal C}^\star_{Ab}(p\land \lnot p) = \frac{1}{3}$ whereas 
${\cal C}^\star_{Ab}(q)=\frac{2}{3}$.
This is counter-intuitive as $p\land \lnot p$ is inconsistent, whereas 
$q$ is consistent.
For the culpability measures ${\cal C}^d$ and ${\cal C}^c$, the situation is comparable, as ${\cal C}^d_{Ab}(\phi)= 1$ and ${\cal C}^c_{Ab}(\phi)=\frac{1}{5}$ for any $\phi\in \{p\land \lnot p,q\}$, i.e.,\ these measures do not distinguish between the self contradictory formula $p\land \lnot p$ and the consistent formula $q$.
\end{example}

We will see below that many ranking-based semantics, including the categoriser semantics, avoid this behaviour (Proposition \ref{prop:ranking-based-semantics:avoid:falsity}).

We formalize this kind of behaviour, or rather, the avoidance of it, by the following postulate of \emph{falsity}:

\begin{definition} 
A semantics $\sigma$ satisfies \emph{falsity} if for any $\ABF=\tup{\mathfrak{L},\Gamma,Ab,\sim}$, if $\Gamma\cup \{\phi\}\vdash {\sf F}$, then for any $\psi\in Ab$, $\sigma_{\ABF}(\phi)\leq \sigma_{\ABF}(\psi)$.
\end{definition}

This postulate is satisfied by a large family of ranking-based semantics (recall Remark \ref{remark:mootony}).

\begin{proposition}\label{prop:ranking-based-semantics:avoid:falsity}
A ranking-based semantics $\sigma$ satisfies falsity if it satisfies monotony.
\end{proposition}
\begin{proof}
Consider a $\phi\in Ab$ s.t.\  $\Gamma\cup \{\phi \} \vdash {\sf F}$.
Then, $\Gamma\vdash {\sim} \phi$. Thus, for any $\Delta\subseteq Ab$, $\Delta$ attacks $\{\phi\}$. So, $\{\phi\}^{-}\supseteq \Delta^{-}$ for any $\Delta\subseteq Ab$. With monotony, for any such $\Delta$ (and thus for any $\{\psi\}\subseteq Ab$ in particular), $\sigma_{\ABF}(\phi)\leq \sigma_{\ABF}(\Delta)$.
\end{proof}

Another property related to the postulate of falsity is that of \emph{self-contradiction} \cite{DBLP:conf/jelia/MattT08}, requiring that self-attacking arguments are ranked lower than any other arguments.
It is clear  that self-contradiction ensures falsity. The self-contradiction postulate has been shown to hold for the nh-categoriser semantics \cite{DBLP:conf/aaai/YunVC20}, and the ranking-based semantics from \cite{DBLP:conf/jelia/MattT08}. 

In this section, we have shown that ranking-based semantics applied to argumentation frameworks generated on the basis of simple contrapositive ABFs induce culpability measures, and thus have a clear logical meaning. Furthermore, we show that, in contradistinction to many other culpability measures, ranking-based semantics satisfy the postulate of falsity. Finally, our method is applicable to a wider variety of settings than existing culpability measures, since we allow for strict premises and any contrapositive logic.

\section{Logical Argumentation Based on Conclusion-Support Pairs}
\label{sec:other:structured:arg}

So far, we have considered the assumption-based argumentation framework, which is one among many formalisms for structured argumentation.
In this section, we show that our results generalize to other settings based on arguments conceived of support-conclusion pairs and using several different attack forms, such as  presented in \cite{arieli2015sequent,besnard2018review}.

\subsection{A Glimpse at Logical Argumentation}

We now consider argumentation frameworks constructed on the basis of arguments conceived of as assumption-conclusion pairs
(see, e.g, \cite{ArStr15argcomp,BesHun18})\footnote{See, e.g.,
\cite{arieli2021logic} for a more detailed comparison of the formalisms. While in such systems typically only defeasible assumptions are considered, recently strict assumptions have been included in \cite{Arieli_et_al_KR_2021}.}, which is another popular approach to logical argumentation.

\begin{definition}[\textbf{Argument}]
\label{def:argument}
Given a logic ${\mathfrak L} = \tup{{\cal L},\vdash}$ and a set of strict premises $\Gamma\subseteq {\cal L}$, an {\em argument\/} is a pair $a = \tup{\Delta,\psi}$, where $\Delta$ is a finite set of ${\cal L}$-formulas and $\psi$ is an ${\cal L}$-formula, 
such that $\Gamma\cup\Delta \vdash \psi$. We denote the set of all arguments by \( \mathsf{Arg}(\Gamma) \).
 \end{definition}

In the following, we shall
denote arguments by $a,b,c$, etc., possibly primed or indexed.  

\begin{itemize}
\item Given an argument $a = \tup{\Delta,\psi}$, we  call $\Delta$ the {\em support set\/} 
 of $a$, and $\psi$ the {\em conclusion\/} (or the {\em claim}) of $a$ (denoted by ${\sf Conc}(A)$).

\item The set of arguments whose supports are subsets of $Ab$ is denoted by ${\sf Arg}(\Gamma,Ab)$. 
That is: ${\sf Arg}(\Gamma,Ab) = \{a \in {\sf Arg}(\Gamma) \mid \mathsf{Sup}(a) \subseteq Ab\}$. 

\end{itemize}

\begin{remark}
\label{note:sequent-based-arg}
An alternative notation for an argument $\tup{\Delta,\psi}$ is $\Delta \Rightarrow \psi$ (where $\Rightarrow$ is a new 
symbol, not appearing in the language of $\Delta$ and $\psi$).
The latter resembles the way sequents are denoted in the context 
of proof theory \cite{Gen34}.
This notation is frequently used in {\em sequent-based argumentation\/} 
(see, e.g., \cite{ArStr15argcomp,ArSt19TCS}) to emphasize the fact that the only requirement on $\Delta$ and $\psi$ 
to form an argument is that the latter follows, according to the base logic, from the former.
\end{remark}

Various attack forms have been studied in the literature.
In logical argumentation, it is often assumed that attacks among arguments are  usually expressed in terms of {\em attack rules\/} (w.r.t. the underlying logic).
Table~\ref{tab:attack-rules} gives some of these attack relations.
Notice that we restrict attention to support-based attacks. 
Other attack rules between arguments are described, e.g. in \cite{GorHun11,BesHun18,ArStr15argcomp}.  
\begin{table*}[h]
{\centering
{\footnotesize
\def\arraystretch{1.5}
\begin{tabular}{|l|l|c|c|c|}
\hline
\rowcolor{black!10} {\bf Rule Name} &  {\bf Acronym} &  {\bf Attacking Argument} &  {\bf Attacked Argument} &  {\bf Attack Conditions}
\\ \hline
Defeat & Def & $\tup{\Delta_1,\psi_1}$ & $\tup{\Delta_2,\psi_2}$ & 
                            $ \Gamma\cup \psi_1 \vdash  \sim\!\bigwedge\!\Delta_2$ \\
\hline
Direct Defeat & DirDef & $\tup{\Delta_1,\psi_1}$ & $\tup{\{\delta_2\} \cup \Delta'_2,\psi_2}$ & 
                            $\Gamma, \psi_1 \vdash \sim\!\delta_2$ \\
\hline
Undercut & Ucut & $\tup{\Delta_1,\psi_1}$ & $\tup{\Delta'_2 \cup \Delta''_2,\psi_2}$ & 
                            $\psi_1 \equiv_\Gamma \sim\!\bigwedge\!\Delta'_2$ \\
\hline
Canonical Undercut & CanUcut & $\tup{\Delta_1,\psi_1}$ & $\tup{\Delta_2,\psi_2}$ & 
                            $\psi_1 \equiv_\Gamma \sim\!\bigwedge\!\Delta_2$ \\
\hline
Direct Undercut & DirUcut & $\tup{\Delta_1,\psi_1}$ & $\tup{\{\delta_2\} \cup \Delta'_2,\psi_2}$ & 
                            $\psi_1 \equiv_\Gamma \sim\!\delta_2$ \\
\hline

\end{tabular}}
\caption{\fontsize{9}{1}Some attack rules in the context of the strict premises $\Gamma$. We let $\phi \equiv_\Gamma \psi$ iff  $\Gamma\cup \{\phi\} \vdash \psi$ and $\Gamma\cup\{ \psi\} \vdash \phi$.
We assume that the support sets of the attacked arguments are 
nonempty (to avoid attacks on theorems).}
\label{tab:attack-rules}
}
\end{table*}

Rules like those specified in Table~\ref{tab:attack-rules} form attack schemes that are applied to particular arguments 
according to the underlying logic. For instance, when classical logic is the underlying formalism, the attacks of $\tup{\{p\},p}$ 
on $\tup{\{\neg p\},\neg p}$ and of $\tup{\{\neg p\},\neg p}$ on $\tup{\{p \wedge q\},p}$ are obtained by applications of the Defeat rule (or other rules
in Table~\ref{tab:attack-rules}).
When an attack rule ${\cal R}$ is applied, we  say that its attacking argument {\em ${\cal R}$-attacks\/} 
the attacked argument. Notice that assumption-based argumentation makes use of the attack form of direct defeat.

\subsection{Considerations on Ranking-based Semantics}

One of the benefits of assumption-based argumentation is the fact that the resulting argumentation frameworks are finite when the set of assumptions $Ab$ is finite.
To the best of our knowledge, all ranking-based semantics are defined only for finite argumentation frameworks, and therefore this property of assumption-based argumentation ensures that ranking-based semantics are applicable to argumentation frameworks constructed in assumption-based argumentation.
This stands in contrast with the systems based on support-conclusion pairs, which even for finite sets of assumptions give rise to infinite sets of arguments (as e.g.\ the argument $\langle \emptyset,p\lor \ldots \lor p \rangle$ will be part of any argumentation framework with $p\lor \ldots \lor p$ having length $n$ for any $n\in\mathbb{N}$).
One way to circumvent this is found in 
\cite{amgoud2015argumentation}.
In that paper, Amgoud and Ben-Naim assume that there exists a well-order $w$ on ${\cal L}$ and restrict attention to the following set of arguments:\footnote{Amgoud \& Ben-Naim \cite{amgoud2015argumentation} require that arguments are additionally consistent and minimal, but we do not generally require this.}
\[\Args^{\sf fn}(\Gamma,Ab) = \{\langle \Delta,\phi\rangle\in \Args(\Gamma,\Ab)\mid \phi= \min_w\{\phi'\mid \phi'\equiv \phi\}\}\]

Thus, the well-ordering $w$ allows selecting one among the infinitely many equivalent conclusions $\phi$.

The first observation we make is that excluding inconsistent arguments, as is done in e.g.\ \cite{amgoud2013ranking,besnard2009}, makes the resulting frameworks useless for obtaining culpability measures, for the simple reason that arguments with inconsistent premises are not allowed, and thus the argumentation framework will not contain arguments that conclude an inconsistent premise. We will therefore not filter out inconsistent arguments.

\begin{example}
Consider $Ab=\{p\land \lnot p,q,\lnot q\land r\}$.
Then, any argument using $p\land \lnot p$ in the support set will have an inconsistent support and thus not be assigned any ranking.
\end{example}

We now show that a wide variety of attack rules ensure satisfaction of postulates for culpability measures.

Given a set of strict premises $\Gamma$ and a set of defeasible premises $Ab$, we apply  a ranking-based semantics over the argumentation framework ${\sf AF}=(\Args^{\sf fn}(\Gamma,\Ab),{\cal R})$, in which the arguments $\Args^{\sf fn}(\Gamma,\Ab)$ are obtained as defined above and the attacks ${\cal R}$ are obtained by applying a set of attack rules 
${\sf R}\subseteq \{\mathrm{Def},\mathrm{DirDef},\mathrm{Ucut},\mathrm{CanUcut},\mathrm{DirUcut}\}$.
Recall that we denote the resulting ranking by $\sigma_{\sf AF}$. 

\begin{proposition} \label{prop:4}
Let some $\Gamma,Ab\subseteq {\cal L}$, ${\sf R}\subseteq \{\mathrm{Def},\mathrm{DirDef},\mathrm{Ucut},\mathrm{CanUcut},\mathrm{DirUcut}\}$,
$\Args=\Args^{\sf fn}(\Gamma,\Ab)$, and 
${\sf AF}=(\Args,{\sf R})$  be given. 
Any ranking-based semantics $\sigma$ that satisfies counter-transitivity, void precedence and void \baserank satisfies the following properties:

\begin{description}

\item[Freeness.] If $\Psi\subseteq \Free(\Gamma,Ab)$\footnote{Recall that ${\sf FREE}(\Gamma,\Ab)={\sf Free}(\tup{\mathfrak{L},\Gamma,Ab,\sim})$.}, then $\sigma_{\sf AF}(\langle \Psi,\psi\rangle)\geq \sigma_{\sf AF}(\langle \Delta,\delta\rangle)$ for any $\langle \Delta,\delta\rangle, \langle \Psi, \psi \rangle \in \Args$ with $\Delta\neq\emptyset$.

\item[Dominance.] Where $\psi, \phi \in \Ab$, if $\Gamma,\psi\vdash \phi$ and $\Gamma \cup\psi\not\vdash {\sf F}$, then $\sigma_{\sf AF}(\langle\{\psi\}, \psi'\rangle)\geq \sigma_{\sf AF}((\langle\{\phi\},\phi'\rangle)$ for all $\langle \{\phi\}, \phi' \rangle, \langle \{\psi\}, \psi'\rangle \in \mathcal{A}$.
\end{description}

\begin{description}
\item[Blame.] 
If $\phi\in\bigcup{\sf MIC}(\Gamma,Ab)$, then $\sigma_{\sf AF}(\langle\{\phi\}, \phi'\rangle) \:<\: \basescore(\sigma)$ for all $\langle \{\phi\}, \phi'\rangle \in \mathcal{A}$.
\end{description}

Furthermore, if $\sigma$ satisfies void \baserank, it satisfies:

\begin{description}
\item[Consistency.] If 
$\Gamma\cup Ab\not\vdash {\sf F}$,
then \( \sigma_{\mathsf{AF}}(\langle \Phi,\phi \rangle) = \basescore(\sigma) \) for every \( \langle \Phi,\phi \rangle \in \mathcal{A} \).

\end{description}
\end{proposition}

\begin{proof}
The following lemma will be useful.
\begin{lemma} \label{lemma:4}
Where \( \langle \Psi, \psi \rangle , \langle \Phi, \phi \rangle \in \mathcal{A}\), if for every \( \langle \Delta, \delta \rangle \in \mathcal{A} \) that attacks \( \langle \Psi, \psi \rangle \) there is an \( \langle \Delta, \delta' \rangle  \in \mathcal{A}\) that attacks \( \langle \Phi,\phi \rangle \), then 
\( \sigma_{\sf AF}(\langle \Phi, \phi \rangle) \le \sigma_{\sf AF}(\langle \Psi, \psi \rangle) \).
\end{lemma}

Note that for each such \( \langle \Delta,\delta \rangle \) and \( \langle \Delta, \delta'  \rangle \) we have \( \langle \Delta,\delta \rangle^{-} = \langle \Delta, \delta' \rangle^{-} \). So, by monotony (which is implied by counter-transitivity), \( \sigma_{\sf AF}(\langle \Delta, \delta \rangle) = \sigma_{\sf AF}(\langle \Delta, \delta' \rangle) \).
Thus, by counter-transitivity, \( \sigma_{\sf AF}(\langle \Phi,\phi \rangle) \le \sigma_{\sf AF}(\langle \Psi, \psi \rangle) \), which proves our lemma.

We now prove the properties of Proposition~\ref{prop:4}.

For {\it Freeness}, suppose \( \Psi \subseteq \mathsf{Free}(\Gamma, \Ab) \) and \( \langle \Psi,\psi \rangle \in \mathcal{A} \). Consider some \( \langle \Delta, \delta \rangle  \in \mathcal{A}\) with \( \emptyset \neq \Delta \).
Suppose \( \langle \Phi, \phi \rangle \) attacks \( \langle \Psi,\psi \rangle \) with respect to any of the attack rules in \( \mathsf{R} \).
Since \( \Psi \subseteq \mathsf{Free}(\Gamma, \Ab) \), \( \Gamma\cup \Phi \vdash \mathsf{F} \).
So, \( \Gamma\cup \Phi \vdash \sim\!\epsilon \) for any \( \epsilon \in \Delta \). Thus, there is a \( \epsilon' \) for which \( \epsilon' \equiv_\Gamma \epsilon \) and \( \langle \Phi,\sim\!\epsilon' \rangle \in \mathcal{A} \).
We note that \( \langle \Phi, \sim\!\epsilon' \rangle \) attacks \( \langle \Delta, \delta \rangle \) with respect to any of the attack rules in \( \mathsf{R} \). 
Thus, by Lemma~\ref{lemma:4},  \( \sigma(\langle \Psi, \psi\rangle) \ge \sigma(\langle \Delta,\delta \rangle) \).

\noindent

For {\it Dominance}, suppose \( \langle \Delta, \delta \rangle \) attacks \( \langle \{ \phi \}, \phi' \rangle \). So, (1) \( \Gamma\cup\{ \delta\} \vdash {\sim} \phi \) (in case of Def and DirDef) or (2) \( \delta \equiv_{\Gamma} {\sim}\phi \) (in case of Ucut, ConUcut and DirUcut). Since \( \Gamma\cup \{\psi\} \vdash \phi \), by contraposition \( \Gamma\cup\{ {\sim}\phi\} \vdash {\sim}\psi \).
So, \( \Gamma\cup\{ \delta\} \vdash {\sim} \psi \) by the transitivity of \( \vdash \). So \( \langle \Delta,\delta \rangle \) also attacks \( \langle \{ \psi \}, \psi' \rangle \) in case (1).
Also, there is a \( \epsilon \) for which \( \epsilon \equiv_{\Gamma} \sim\!\psi \) and \( \langle \Delta, \epsilon \rangle \in \mathcal{A} \). In case (2), \( \langle \Delta, \epsilon \rangle \) attacks \( \langle \{ \psi \}, \psi' \rangle \).
By Lemma~\ref{lemma:4}, \( \sigma_{\sf AF}(\langle \{ \phi \}, \phi' \rangle) \ge \sigma_{\sf AF}(\langle \{ \psi \}, \psi' \rangle) \).

\noindent
For \textit{blame}, 
suppose  $\phi\in\bigcup{\sf MIC}(\Gamma,Ab)$. Notice that $\langle \emptyset,\psi\rangle\in {\cal A}$ for any $\psi\in {\cal L}$ s.t.\ $\Gamma\vdash \psi$. Furthermore, inspection of all the attack rules confirms that $(\langle \emptyset,\psi\rangle)^{-}=\emptyset$. Since $\phi\in\bigcup{\sf MIC}(\Gamma,Ab)$, there is a $\Delta\in {\sf MIC}(\Gamma,Ab)$ s.t.\ $\Gamma\cup\Delta\vdash {\sf F}$ and $\phi\in\Delta$. Thus, $\Gamma\cup\Delta\vdash {\sim} \phi$ (with contraposition), which means $\langle \Delta,\lnot \phi^\star\rangle\in {\cal A}$ (for $\phi^\star\equiv_\Gamma \phi$) attacks $\langle \{\phi\},\phi'\rangle\in {\cal A}$ and thus $(\langle \{\phi\},\phi\rangle)^{-}\neq\emptyset$. With void precedence and void \baserank, this means $\sigma_{\sf AF}(\langle \{\phi\},\phi\rangle)< \sigma_{\sf AF}(\langle \emptyset,\psi\rangle)=  \basescore(\sigma)$.

\noindent
For {\it Consistency}, suppose 
$\Gamma\cup Ab\not\vdash {\sf F}$
and \( \langle \Phi,\phi \rangle \in \mathcal{A}\). 
Then, \( \langle \Phi,\phi \rangle ^{-} = \emptyset\) and by void \baserank, \( \sigma_{\mathsf{AF}}(\langle \Phi,\phi \rangle) = \basescore(\sigma) \).
\chrrminv{For \emph{Consistency}, suppose $Ab={\sf Free}(\Gamma,Ab)$.
Then, there are no ${\cal R}$-attacks between arguments, and thus $\sigma_{\sf AF}(\langle\{\phi\},\phi\rangle)= \basescore(\sigma)$ for every $\phi\in Ab$ with void \baserank.}
\end{proof}

\section{Related Work}\label{sec:rel:work}

Interestingly enough, one of the earliest ranking-based semantics, namely the categoriser semantics \cite{besnard2001logic} has been proposed in the context of logical argumentation. Since then, ranking-based semantics have been studied in breath and depth, but from an abstract level.  \cite{besnard2001logic} does not use Dungean argumentation frameworks, but rather is based on argumentation trees for a given conclusion.  
To the best of our knowledge, ranking-based semantics have been applied to structured argumentation in only one other work, namely \cite{amgoud2015argumentation}.
One difference is that we do not assume that arguments have a consistent or minimal set of supports. In our paper, we do not filter out inconsistent arguments, as we need to be able to assign culpability measures to (potentially inconsistent) premises. In the work of Amgoud and Ben-Naim, this is not possible, as an argument like $\langle \{p\land\lnot p\},p\land\lnot p\rangle$ is not allowed in their framework. 
Another benefit is that we allow for strict premises.
A final  difference with \cite{amgoud2015argumentation} is that Amgoud and Ben-Naim adopt only the direct undercut attack rule between argument, whereas we consider various attack forms.

Culpability measures have been investigated in several proposals.
To the best of our knowledge, we have studied all postulates proposed for culpability measures in the literature.
In future work, we want to compare ranking-based semantics with more involved culpability measures, such as the \emph{tacit culpability measures} from \cite{ribeiro2021consolidation}.

\section{Conclusion and Future Work}\label{sec:concl}

This paper contains a principled investigation of ranking-based semantics applied to structured argumentation.
Rather than having the last word on this topic, we hope to have established that ranking-based semantics for structured argumentation are meaningful and useful, and thus open new ground for research.
We believe this work can be extended in various directions, e.g.\ to ranking-semantics over sets of attacking logical arguments, other frameworks for structured argumentation \cite{Bondarenko1997,vcyras2021computational,arieli2015sequent,Prakken2010,heyninck2017revisiting,AH21,vcyras2016aba,DBLP:journals/flap/HeyninckS21,DBLP:journals/ijar/ArieliH21}, which give a natural account of the notion of so-called intrinsic strength of arguments often assumed in ranking-based logics \cite{baroni2019fine}.

\bibliographystyle{plain}
\bibliography{rankingSCABFs}

\end{document}